\documentclass[fleqn,10pt,twocolumn]{AROB-ISBC-SWARM23}

\usepackage{amsmath} 
\usepackage{amssymb}  
\usepackage{graphicx}
\usepackage{mathabx} 
\usepackage{caption}
\usepackage{subcaption}

\newcommand{\R}{{\mathbb{R}}}

\newcommand{\eg}{\emph{e.g.,\,}}
\newcommand{\ie}{\emph{i.e.,\,}}
\DeclareMathOperator*{\argmin}{arg\,min}
\newcommand{\sech}{\mathrm{sech}}
\newcommand{\ve}[2][]{{\mathbf{#2}}_{#1}}
\newcommand{\vd}[2][]{\dot{\mathbf{#2}}_{#1}}

\usepackage{amsthm}
\theoremstyle{plain}
\newtheorem{thm}{Theorem}[]
\newtheorem{definition}{Definition}[]
\newtheorem{remark}{Remark}[]
\newtheorem{lemma}{Lemma}[]

\allowdisplaybreaks

\title{tinySLAM-based exploration with a swarm of nano-UAVs}

\author{Mattias Vikgren${}^{\dagger,1}$ and Johan Markdahl${}^{1}$}
\speaker{Mattias Vikgren}

\affils{${}^{1}$The Swedish Defence Research Agency, Stockholm, Sweden\\
(E-mail: mattias.vikgren@foi.se, johan.markdahl@foi.se)}
\abstract{%
This paper concerns SLAM and exploration for a swarm of nano-UAVs. The laser range finder-based tinySLAM algorithm is used to build maps of the environment. The maps are synchronized using an iterative closest point algorithm. The UAVs then explore the map by steering to points selected by a modified dynamic coverage algorithm, for which we prove a stability result. Both algorithms inform each other, allowing the UAVs to map out new areas of the environment and move into them for exploration. Experimental findings using the nano-UAV Crazyflie 2.1 platform are presented. A key challenge is to implement all algorithms on the hardware limited experimental platform. }


\begin{document}

\maketitle


\section{INTRODUCTION}

Autonomous swarms of flying robots is an emerging technology with promising potential. There is an advancing trend towards smaller robots due to developments in hardware. In particular, nano-robots are well-suited to travel in narrow and cluttered indoor environments. They are also safe to operate around humans. Applications that leverage these strengths include tasks like surveillance \cite{saska2016swarm}, search-and-rescue operations \cite{mcguire2019minimal}, and gas leak localization \cite{duisterhof2021sniffy}. To realize such tasks in autonomous swarms, a number of control algorithms must be employed, including GPS-denied localization, mapping, and exploration. In this paper we demonstrate such capabilities in experiments with a swarm of Crazyflie 2.1 which are unmanned aerial vehicles (UAVs)  of size 27 g, 1 dm$^2$.

We couple a low memory and low computational complexity laser range finder (LRF) based simultaneous localization and mapping (SLAM) algorithm \cite{steux2010tinySLAM} with a modified swarm exploration algorithm \cite{ganganath2016distributed} and implement them on-board the hardware limited Crazyflie 2.1 nano-UAV platform. The two algorithms synergize; the SLAM maps the environment and the exploration moves the swarm to previously occluded areas, which the SLAM algorithm can survey. This behavior is realized in an experiment in two adjacent rooms. These capabilities have not previously been demonstrated using on-board calculations in a nano-UAV swarm without an external positioning system \cite{coppola2020survey}.

\subsection{Literature review}

The survey paper \cite{coppola2020survey} concerns swarming with micro air vehicles. A key distinction is made between centralized control \emph{vs.} distributed control. Likewise there is a difference between position sensing that rely on an external system such as GPS or a motion caption system \emph{vs.} on-board sensing. The case of distributed control using only on-board sensing is the most difficult, with state-of-the-art number of swarm members being in the single digits \cite{coppola2020survey}.

Low complexity distributed control/on-board sensing algorithms used for maze solving with tiny robots are referred to as bug algorithms \cite{mcguire2019comparative}. The work \cite{mcguire2019minimal} implements a wall following bug algorithm with a swarm of Crazyflie 2.0 UAVs. The UAVs travel through a complex environment along their preferential direction and later return to their departure point by following a beacon. The work \cite{duisterhof2021sniffy} uses a smell sensor on a swarm of Crazyflie 2.1 UAVs to detect and localize gas leaks. Compared to \cite{mcguire2019minimal,duisterhof2021sniffy} we implement tweaked, light versions of more complex algorithms on the same Crazyflie hardware. 

The area coverage task requires robots to move their sensors to cover an area of interest. Individual coverage algorithms have a long history (see \eg \cite{gabriely2001spanning} and the survey \cite{galceran2013survey}), and there also exist swarm versions (\cite{hazon2005redundancy} is the multi-robot version of \cite{gabriely2001spanning}). A classic exploration algorithm is the mutual information based work \cite{burgard2005coordinated}. That work does not consider flying robots but is still similar in scope to our setup. An algorithm for persistent exploration is given in \cite{ganganath2016distributed}. We prefer a modified version of the algorithm \cite{ganganath2016distributed} to more systematic algorithms like \cite{hazon2005redundancy}. This is due to the algorithm   \cite{ganganath2016distributed} being a feedback law whereas \cite{hazon2005redundancy} is open loop. Moreover, in certain scenarios it can be preferable that the algorithm in \cite{ganganath2016distributed} is unpredictable, as that makes it more difficult for mobile adversarial agents to avoid discovery. 

\begin{figure}[!hbt]
    \centering
    \includegraphics[width=0.25\textwidth,trim=0 1cm 0 1cm, clip]{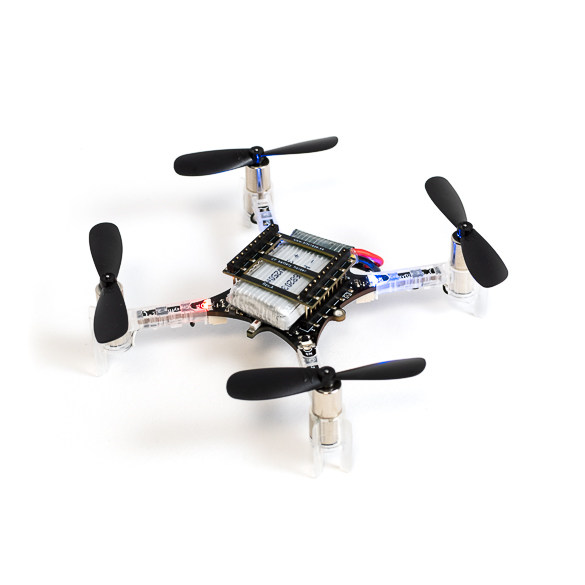}
    \caption{The Crazyflie 2.1. It weighs $27$ g and measures $92 \textrm{ mm}\times92 \textrm{ mm}\times29\textrm{ mm}$.} 
    \label{fig:crazyflie}
\end{figure}

\subsection{Contribution}

The main contribution of this paper is to couple the tiny-SLAM algorithm \cite{steux2010tinySLAM} with an  exploration algorithm \cite{ganganath2016distributed} such that the two algorithms inform each other about the environment that a swarm of nano-UAVs operates in. Distributed algorithms of this complexity that only use on-board calculations and sensing have not previously been implemented on nano-UAV platforms. There are examples of complex swarm algorithms being executed on the Crazyflie 2.1 nano-UAV, see \eg \cite{soria2021predictive}, but they rely on a motion capture system and model predictive control calculations done on a remote computer. We also modify the exploration algorithm \cite{ganganath2016distributed} and prove a stability result for our version.

\section{Model}


All UAVs are quadrotors with planar position and orientation states $(\ve[i]{x},\theta_i)\in{\R}^2\times(-\pi,\pi]$, where $i\in\{1,\ldots,N\}$ refers to the $i$th UAV. A Crazyflie 2.1 can be actuated by position and velocity commands. We mainly use velocity commands, \ie we assign $\vd[i]{x}=\ve[i,x]{u}$ and $\dot{\theta}_i=u_{i,\theta}$. The input signals are saturated, $\|\ve[i,x]{u}\|\leq S_x,\,|u_{i,\theta}|\leq S_\theta$, as a safety measure and to prevent resetting of the Crazyflie platform (in the experiments, we observe that the extended Kalman filter in a Crazyflie is reset to zero if the states change quickly).

The camera's field of view is identified with the forward direction in the UAV's body fixed frame, see Fig. \ref{fig:camera}. The coordinates of the set of points corresponding to the  camera image is a rectangular area centered around a point $\ve[i,c]{x}=\ve[i]{x}+r_c[\cos\theta_i\,\sin\theta_i]^T$, where $r_c$ is a camera parameter. The width of the focus area is $w$ and the length of the focus area is $l$. The field of view is 
${\mathcal{V}}_i=\{\ve{y}\in\R^2:\,|\langle \ve{y}-\ve[i,c]{x},\ve[i,1]{v}\rangle|\leq w/2,
|\langle \ve{y}-\ve[i,c]{x},\ve[i,2]{v}\rangle|\leq l/2\},$ 
where 
$\ve[i,1]{v}=[
\cos\theta_i\,
\sin\theta_i]^T$, 
$\ve[i,2]{v}=[
-\sin\theta_i\,
\cos\theta_i]^T$.
Note that the cameras are used for exploration, but not for SLAM.

\section{Collision avoidance}

Each Crazyflie 2.1 is equipped with six LRFs that are used to detect obstacles, walls, the ceiling, and the floor. Four are directed to the front, back, left, and right of the UAV. The UAV is programmed to move away from any obstacle it detects when the distance to the obstacle is lower than a preset threshold. Because the UAVs are controlled to move with the camera sensor directed forward, the front LRF detects any obstacle in its path. This type of naive collision avoidance algorithm may lead to UAVs getting stuck, which can be avoided by path planning or by only moving towards points that are in the visible line of sight.

The multi-agent collision avoidance we use in this paper is similar to the potential-field based method in \cite{ganganath2016distributed}, and indeed in many of other work, see \eg  \cite{olfati2006flocking}. The key difference is that rather than assuming that all agents can sense the relative localization $\ve[i]{x}-\ve[j]{x}$, each agent broadcasts its position $\ve[i]{x}$ in a map they have built using SLAM. This requires the maps to be synchronized, which we achieve by matching point clouds using an iterative closest point algorithm. The knowledge of $\ve[i]{x}$ and $\ve[j]{x}$ allows an agent to calculate  $\ve[i]{x}-\ve[j]{x}$, which is used for collision avoidance and swarm algorithms.

\section{SLAM}

The Crazyflie 2.1 platform restricts the choice of SLAM implementation due to its severe limitations in terms of computational power, sensory information, and  memory. Therefore, the tinySLAM algorithm \cite{steux2010tinySLAM} is chosen for its lightweight implementation and compatibility with LRF sensors. Due to the restrictions on the available computational resources, the algorithm is used without the recommended particle filter as a single hypothesis SLAM method.

\begin{figure}
    \centering
    \includegraphics[width=0.5\textwidth]{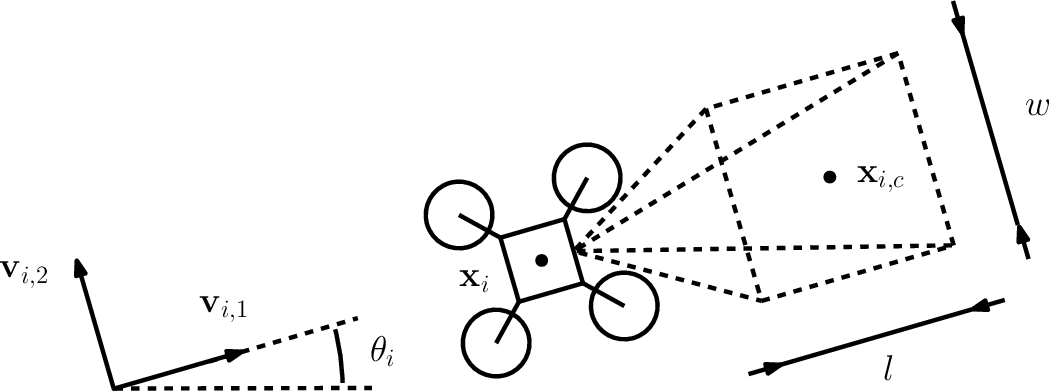}
    \caption{Camera model. The vectors $\ve[i,1]{v}$ and $\ve[i,2]{v}$ are the basis vectors of the $i$th body frame  expressed in the world frame which is aligned with the horizontal and vertical axes.}
    \label{fig:camera}
\end{figure}

The tinySLAM map representation has been adopted to the Crazyflie 2.1 system. The SLAM map $\ve[i]{n}$ of UAV $i$ is of occupancy grid type and discretizes the environment into equally sized square cells, $10\textrm{ cm}\times10\textrm{ cm}$ in size. Each cell is associated with an 8-bit integer $(\ve[i]{n})_{ab}$, \ie a number in the range $0$ to  $255$, which represents a degree of confidence in whether a cell is occupied by an obstacle or not. The number $0$ represents occupied and $255$ free. The cells are initialized at $127$, representing uncertainty about their state. The map is a $100\times100$ matrix, occupying $10^4$ out of a total of $5\cdot10^4$ bytes in the Crazyflie 2.1 memory.

The map update is based on Bresenham's line algorithm but has been modified by the creator of the tinySLAM algorithm to facilitate faster convergence of the scan-matching method used for localization \cite{steux2010tinySLAM}. The main modification is what the authors describe as the `hole function', that blurs measured obstacles to the neighboring cells in the direction of the scan. This gives the walls a thickness larger than one cell. The update of the cell values can be tuned using a confidence parameter $\alpha\in\{0, \ldots,255\}$ and a parameter that sets the hole width. The confidence parameter $\alpha$ defines how the update of each cell weighs the state of the existing map compared to the new measurement $y\in\{0,255\}$ as
\begin{align}
(\ve[i]{n})_{ab}(k+1) = ((255 - \alpha)(\ve[i]{n})_{ab}(k) + \alpha y)/255.
\label{eq:map_update}
\end{align}

The tinySLAM algorithm utilizes a Monte Carlo scan matching localization method. The scan matching compares the latest laser scan to the current map in different positions decided by a random walk starting in the current estimated position. The position that the scan matcher finds to best fit to the map is fed into the extended Kalman filter used for sensor fusion and state estimation of the UAV.

\section{Relative localization}
Multi-agent systems require relative localization for data fusion and collision avoidance between the agents. Our method for relative localization is based on an iterative closest point (ICP) algorithm commonly used for map merging or scan matching in multi-agent SLAM systems \cite{sobreira2019Map-matching}.  The ICP algorithm is an iterative point registry method which tries to find the rigid transformation that minimizes the Euclidean distance between two point sets with unknown correspondence. The two point sets correspond to obstacle coordinates in two maps $\ve[i]{n}$ and $\ve[j]{n}$ created by two different UAVs. The ICP algorithm provides an estimate of the transformation between the two agent’s local maps. It is used to transform the locally estimated position being communicated between the agents for relative localization. 

Due to the limited communication capabilities and memory available on the Crazyflie 2.1 we suggest an approach for selecting a few dozen points representing the obstacles in map $\ve[i]{n}$ created by the SLAM algorithm, rather than using all points marked as obstacles in the map. These points are extracted by iterating over the map to find obstacle cells in the map. The coordinates of the obstacle cell is added to the point set if there are no points already added to the set within a certain distance threshold. The threshold has to be tuned to find equally spaced points covering the entire area that has been mapped by the SLAM algorithm without consuming too much memory. 

The transform is found through a matching and transformation stage, which is repeated until convergence. The matching stage is started with finding the correspondence between the two point sets. Each point $\ve[i,k]{p}$ in the set ${\mathcal{S}}_i$ of UAV $i$ is associated with the point $\ve[j,l]{q}$ in the set ${\mathcal{S}}_j$ of UAV $j$ having the minimum Euclidean distance  $\|\ve[i,k]{p}-\ve[j,l]{q}\|$. If the distance is larger than a certain tolerance, then the point is ignored and it is seen as an outlier in this iteration. 

The transformation stage is based on a least square solver to find the optimal transformation between the two point sets. First calculate the centroids $\widebar{\ve{p}}_i$ and $\widebar{\ve{q}}_j$ 
\begin{align*}
\widebar{\ve{p}}_i = \frac{1}{K} \sum_{k=1}^K\ve[i,k]{p},\quad
\widebar{\ve{q}}_j = \frac{1}{K} \sum_{l=1}^K\ve[j,l]{q},
\end{align*}
and the distance of each point to the centroid, $\ve[i,k]{p}^\prime$ and $\ve[j,l]{q}^\prime$,
\begin{align*}
\ve[i,k]{p}^{\prime} = \ve[i,k]{p} - \widebar{\ve{p}}_i,\quad
\ve[j,l]{q}^{\prime} = \ve[j,l]{q} - \widebar{\ve{q}}_j.
\end{align*}
Use the distances to the centroids to form matrices
\begin{align*}
\ve[i]{P}^{\prime}=\begin{bmatrix}
\ve[i,1]{p}^\prime\,
\ve[i,2]{p}^\prime\,
\ldots
\ve[i,K]{p}^\prime
\end{bmatrix},\quad
\ve[j]{Q}^\prime=\begin{bmatrix}
\ve[j,1]{q}^\prime\,
\ve[j,2]{q}^\prime\,
\ldots
\ve[j,K]{q}^\prime
\end{bmatrix}.
\end{align*}
Then calculate the cross-variance matrix $\ve[ij]{H}$,
\begin{align*}
\ve[ij]{H} = \ve[i]{´P}^{\prime} \ve[j]{Q}^{\prime\mathsf{T}}\in \R^{2\times 2}.
\end{align*}
The matrix $\ve[ij]{H}$ yields the rotation angle $\vartheta_{ij}$ as
\begin{align*}
\vartheta_{ij} = \mathrm{atan}2((\ve[ij]{H})_{12}-(\ve[ij]{H})_{21}, (\ve[ij]{H})_{11}+(\ve[ij]{H})_{22}).
\end{align*}
Now the points can be transformed using the rotation matrix $\ve{R}(\vartheta_{ij})$ and the translational shift $\ve[ij]{t}$, where
\begin{align*}
\ve{R}(\vartheta_{ij}) = 
\begin{bmatrix}
\cos(\vartheta_{ij})     &-\sin(\vartheta_{ij})\\
\sin(\vartheta_{ij})     &\cos(\vartheta_{ij})
\end{bmatrix}, \ve[ij]{t} =  \widebar{\ve{q}}_{j} - \ve{R}(\vartheta_{ij})\widebar{\ve{p}}_{i}.
\end{align*} 
The rotation matrix $\ve{R}(\vartheta_{ij})$ and translation $\ve[ij]{t}$ is applied to points in set ${\mathcal S}_j$ and the next iteration starts from there. 

Each UAV in the swarm has a predetermined neighbor with which it is supposed to synchronize its map. The UAVs that synchronize their maps are chosen such that a chain of transformations are created and communicated. This enables all UAVs to transform coordinates from all other UAVs to its own map coordinates. Note that this allows for a distributed solution, without all-to-all communication. Moreover, the pairs of communicating neighbors could be determined dynamically based on current relative distances rather be predetermined.

\section{Exploration control design}

Part of the exploration control design is taken from the dynamic coverage algorithm \cite{ganganath2016distributed}. However, we modify it for use on a low memory, low computational resources hardware. Moreover, we add an orientation state and design the control law to couple the position and orientation errors.

\subsection{The camera coverage map}


Each robot has two maps of its own. The map $\ve[i]{n}$ created by the SLAM algorithm keeps track of obstacles and walls in the environment. A second map $\ve[i]{m}$ is used to keep track of when points in the environment were last seen by the on-board cameras. The camera coverage map $\ve[i]{m}$ is an $A\times B$ matrix where each element $(\ve[i]{m})_{ab}$ is an 8-bit integer, \ie $\ve[i]{m}\in\{0,1,\ldots,255\}^{A\times B}$.
%

%
%

A discrete time-step $k\in\{1,\ldots,255\}$ is used to time-stamp the map. Each time step has a length $\Delta t$. Let $t$ denote the current time. For $t\in[\Delta t k,\Delta t(k+1))$, the map is updated as $(\ve[i]{m})=k$ for all $\ve{y}$ in the field of view ${\mathcal{V}}_i$. The map $\ve[i]{m}$ can be interpreted as a landscape of time-stamps that keeps track of the age of information. This limits the total time of exploration to $255\Delta t$, but with $\Delta t\approx1$ that is still larger than the battery allows.

%
%

Our approach to exploration is based on \cite{ganganath2016distributed}. However, their algorithm is not designed for nano-UAVs. In particular, they require any pair of agents $(i,j)$ that are sufficiently close to one-another to exchange their maps $\ve[i]{m}$ and $\ve[j]{m}$. The Crazyflie 2.1 communication bandwidth is too limited for this. Instead, we assume that agents exchange only their current position and orientation in a local map. This suffices for collision avoidance and agent $i$ can timestamp its map $\ve[i]{m}$ based on receiving $(\ve[j]{x},\theta_j)$ for all $j$ (or just a subset of neighbors for a distributed solution).

\subsection{Position control design}

The exploration algorithm strives to make the UAV see its desired camera point $\ve[i,c,d]{x}\in{\R}^2$ inside its field of view ${\mathcal{V}}_i$. This is done by solving an optimization problem using the brute force method. The calculation requires  ${\mathcal{O}}(AB)$ operations, where $A\times B$ is the dimension of the map $\ve[i]{m}$. This is the most expensive calculation of our algorithm, but it is only done infrequently when updating $\ve[i,c,d]{x}$ as it enters the field of view ${\mathcal V}_i$. 

As detailed in \cite{ganganath2016distributed}, the next desired camera point $\ve[i,c,d]{x}$ is calculated as
\begin{align}
\ve[i,c,d]{x}&=\argmin_{\ve{y}}(t-(\ve[i]{m})_{ab})(\lambda+(1-\lambda)f_i(\ve{y})),\label{eq:argmin}\\
f_i(\ve{y})&=\exp(-\sigma_1\|\ve{y}-\ve[i,c]{x}\|-\sigma_2\|\ve{y}-\ve[i,c,d]{x}^\prime\|).\nonumber
\end{align}
The function $f_i$ is used to penalize points $\ve{y}$ for being either too far from the UAV's current camera position $\ve[i,c]{x}$ or too far from the previous desired camera point $\ve[i,c,d]{x}^\prime$. The parameter $\lambda\in[0,1]$ determines if the UAVs travel as a group or spread out. In particular, for $\lambda=1$, all the UAVs will calculate the same next desired camera point. The current desired UAV orientation $\theta_{i,d}$ is calculated based on $\ve[i,c]{x}$ and $\ve[i,c,d]{x}$.

The camera position error, heading error, exploration control law, and closed loop system are defined by
\begin{align}
\ve[i,c]{e}:\!&=\ve[i,c,d]{x}-\ve[i,c]{x}=\ve[i,c,d]{x}-\ve[i]{x}-r_c\ve[i,1]{v},\nonumber\\
e_{i,\theta}:\!&=\theta_{i,d}-\theta_i\in(-\pi,\pi],\nonumber\\
\ve[i]{u}:\!&=k_{c}\,{\sech}^2(k_{s}e_{i,\theta})\tanh(k_t\|\ve[i,c]{e}\|)\,\ve[i,c]{e}/\|\ve[i,c]{e}\|,\label{eq:modified}\\
\vd[i]{x}&=\ve[i]{u},\nonumber\\
\vd[i,c]{e}&=-\ve[i]{u}-\dot{\theta}_ir_c\ve[i,2]{v},\nonumber
\end{align}
where $k_t,k_s,k_c\in(0,\infty)$. The control $\ve[i,c]{u}$ is designed to steer the center of its camera field of view, $\ve[i,c]{x}$, to the desired point $\ve[i,c,d]{x}$. The unit vector and $\tanh$ function make $k_c$ easier to tune. The UAVs should rotate before they translate. This allows the camera and LRF to align with the direction of motion. Note that this also excludes the use of a feedforward term in $\ve[i]{u}$ to cancel the $\ve[i,2]{v}$ term in $\vd[i,c]{e}$. To achieve fast convergence to $\theta_d$, we could use a high-gain control law. However, we observe that high angular velocities lead to resetting of the Kalman filter of the Crazyflie 2.1, which interferes with the SLAM algorithm. We hence choose a control design that couples the position and heading errors, where the $\sech$ function is shaped like a Gaussian and makes translational velocities small when $k_s|e_{i,\theta}|\gg0$. 

The control \eqref{eq:modified} differs from that in \cite{ganganath2016distributed} in that it couples the position and orientation error (note that there is no orientation state in \cite{ganganath2016distributed}). To gain intuition for our approach, we could consider a control action that consists of two separate maneuvers. First, a rotation to align the UAV orientation $\ve[i,1]{v}$ with the line that connects $\ve[i,c]{x}$ and $\ve[i,c,d]{x}$. Second, a translation that takes $\ve[i,c]{x}$ to $\ve[i,c,d]{x}$. However, that sequence of maneuvers would be an open loop control which would perform poorly due to errors and disturbances. Our feedback control \eqref{eq:modified} fuses these two separate maneuvers into a single smooth motion. The notion of two separate motions is still helpful; it is also the intuition for the proof of our main theoretical result, Theorem \ref{th:main} in the next section.


\subsection{Heading control design}
\label{sec:theta}
The heading $\theta_i\in(-\pi,\pi]$ or yaw angle of the $i$th UAV is controlled to align with the desired direction of travel $\theta_{i,d}={\mathrm{atan}}2((\ve[i,c]{e})_2,(\ve[i,c]{e})_1)$. The Crazyflie 2.1 platform allows control by  angular velocity commands, $\dot{\theta}_i=u_{i,\theta}$, where $u_{i,\theta}\in{\R}$ is the input. The heading error is 
\begin{align*}
e_{i,\theta}=\theta_{i,d}-\theta_i\in(-\pi,\pi].
\end{align*}

The heading control and closed-loop system is given by
\begin{align}
u_{i,\theta}:=k_{\theta}\frac{\sin e_{i,\theta}}{\sqrt{1+\cos e_{i,\theta}}},\quad
\dot{\theta}_i=u_{i,\theta}.\label{eq:utheta}
\end{align}
The trigonometric functions $\sin$ and $\cos$ remove the discontinuity due to $e_{i,\theta}\in(-\pi,\pi]$. The control \eqref{eq:utheta} is the ${\mathsf{SO}}(2)$ version of the large angle control law on ${\mathsf{SO}}(3)$ proposed in \cite{lee2012exponential}. It achieves its highest gain at $e_{i,\theta}\approx\pi$, leading to fast convergence in the case of large errors. This is desirable since our position control is slow for large rotation errors $e_{i,\theta}$. Note that the control \eqref{eq:utheta} is discontinuous at $e_{i,\theta}=\pi$.

\section{Main results}

\begin{figure*}[h!]
\centering
\includegraphics[trim=3cm 5cm 3cm 5cm, clip,width=1\linewidth]{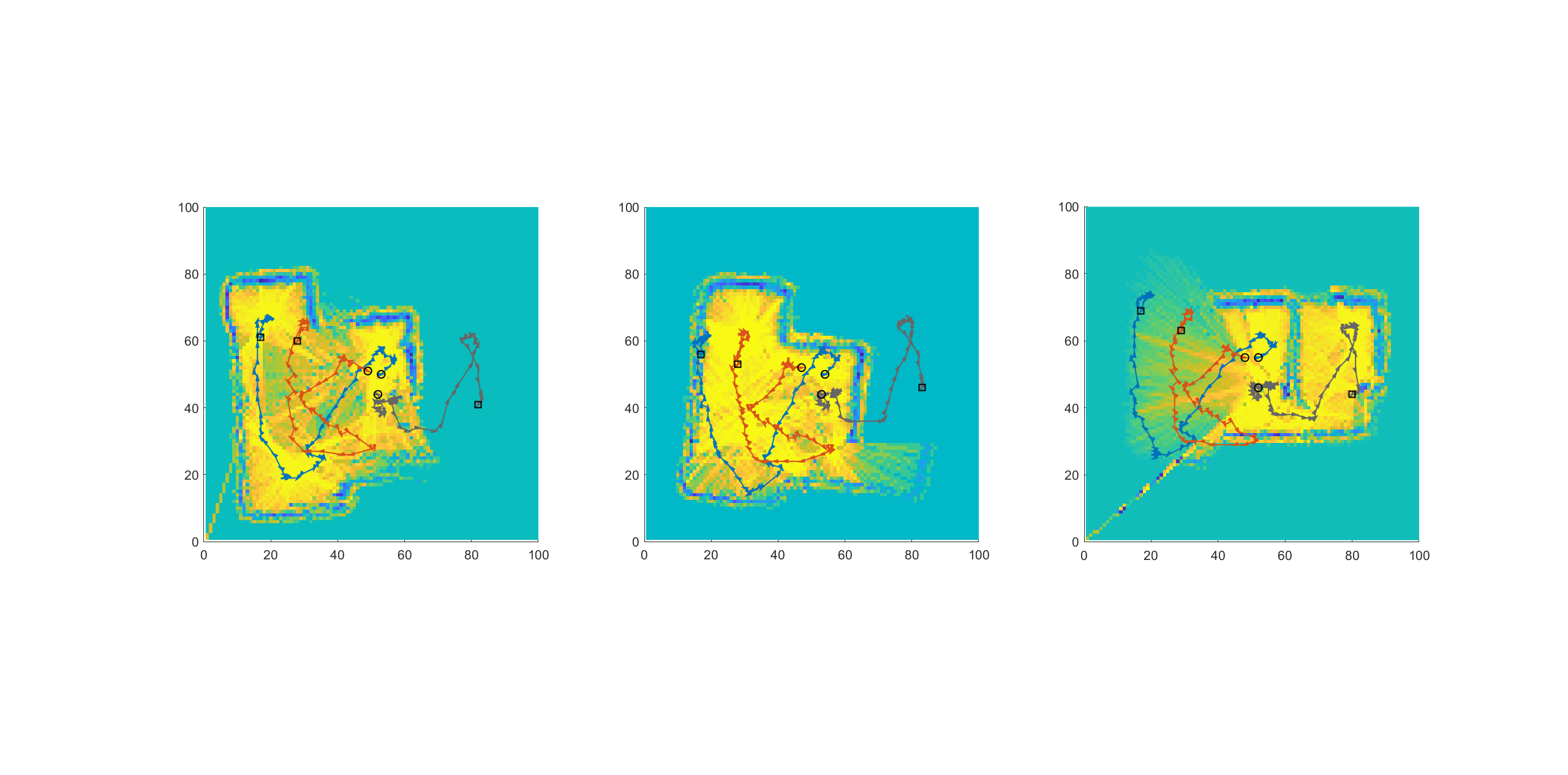}
 \caption{Local views of the environment given by the SLAM maps $\ve[i]{n}$ (heatmaps) and the three UAV position and orientation states (triangles and colored paths). The local perspectives are those of the blue UAV (left), the red UAV (middle) and the gray UAV (right) as can be inferred from the yellow areas of obstacle free grid points surrounding their trajectories. The initial positions of all UAVs are given by black circles and the final position by black squares.}
\label{fig:slam}
\end{figure*}

The results of this paper are: a theoretical proof of stability of the modified exploration algorithm based on \cite{ganganath2016distributed} and an experiment with a nano-UAV swarm with the Crazyflie 2.1 platform. Note that we have run many more experiments than those displayed here. We strive to provide results that guide intuition rather than quantitative characterization.

\subsection{Stability analysis}

A number of theoretical results about the control algorithm on which our approach is based are given in \cite{ganganath2016distributed}. We extend the algorithm of \cite{ganganath2016distributed} to include control of the UAV orientation. The dynamics of the error terms are given by
\begin{align}
\vd[i,c]{e}={}&-k_{c}\,{\sech}^2(k_{s}e_{i,\theta})\tanh(k_t\|\ve[i,c]{e}\|)\,\ve[i,c]{e}/\|\ve[i,c]{e}\|+\nonumber\\
{}&-k_\theta r_c\frac{\sin e_{i,\theta}}{\sqrt{1+\cos e_{i,\theta}}}\begin{bmatrix}
-\sin (\theta_{i,d}-e_{i,\theta})\\
\cos (\theta_{i,d}-e_{i,\theta})
\end{bmatrix},\label{eq:errors}\\
\dot{e}_{i,\theta}={}&-k_\theta\sin e_{i,\theta}/\sqrt{1+\cos e_{i,\theta}}.\label{eq:etheta}
\end{align}
The system has a triangular structure where $\vd[i,c]{e}$ depends on $\ve[i,c]{e}$ and $e_{i,\theta}$, whereas $\dot{e}_{i,\theta}$ depends on $e_{i,\theta}$ but not on  $\ve[i,c]{e}$. We leverage this structure to prove a stability result for our version of the algorithm \cite{ganganath2016distributed}. First we need some theory.

\begin{lemma}[Vidyasagar \cite{vidyasagar1980decomposition}]\label{th:triangular}
The origin $(\ve{0},\ve{0})$ of a system 
\begin{align}
\vd{x}=\ve{f}(\ve{x},\ve{y}),\quad
\vd{y}=\ve{g}(\ve{y}),\label{eq:triangular}
\end{align}
is asymptotically stable if and only if $\ve{0}$ and $\ve{0}$ are asymptotically stable equilibria of $\vd{x}=\ve{f}(\ve{x},\ve{0})$ and $\vd{y}=\ve{g}(\ve{y})$. 
\end{lemma}

The notion of almost global stability of an equilibrium corresponds to the idea of stability and convergence from all initial conditions except for a set that is negligible in size.

\begin{definition}
An equilibrium $\ve{y}\in\mathcal{X}$ of a dynamical system $\vd{x}=\ve{f}(\ve{x})$ on a space $\mathcal{X}$ is \emph{almost globally asymptotically stable} if it is stable and $\lim_{t\rightarrow\infty}\ve{x}(t;\ve[0]{x})=\ve{y}$ for all $\ve[0]{x}\in\mathcal{X}\backslash\mathcal{N}$, where $\mathcal{N}$ is a set of Lebesgue measure zero.
\end{definition}

What follows is our main theoretical result:

\begin{thm}\label{th:main}
The origin $(\ve{0},0)$ is an almost globally asymptotically stable equilibrium of the error dynamics \eqref{eq:errors}--\eqref{eq:etheta}.
\end{thm}

\begin{proof}
The error dynamics \eqref{eq:errors}--\eqref{eq:etheta} are on the triangular form \eqref{eq:triangular}. Moreover, the  two decoupled subsystems,
\begin{align*}
\vd[i,c]{e}|_{e_\theta=0}&=-k_{c}\tanh(k_t\|\ve[i,c]{e}\|)\,\ve[i,c]{e}/\|\ve[i,c]{e}\|,\\
\dot{e}_{i,\theta}&=-k_\theta\sin e_{i,\theta}/\sqrt{1+\cos e_{i,\theta}},
\end{align*}
both have an asymptotically stable equilibrium in $\ve{0}$ and $0$ respectively. By Lemma  \ref{th:triangular}, $(\ve{0},0)$ is an asymptotically stable equilibrium of \eqref{eq:errors}--\eqref{eq:etheta}. In particular, $(\ve{0},0)$ is stable.

It remains to prove almost global attractiveness of $(\ve{0},0)$. We will use the comparison principle \cite{khalil} for a Lyapunov function based on $\ve[i,c]{e}$. It is straightforward to show that $e_{i,\theta}$ converges to $0$ for all initial conditions except $\mathcal{N}:=\{\pi\}$, \ie from almost all initial conditions. Note that  Equation \eqref{eq:etheta} can be solved for the state $e_{i,\theta}$ as a function of time. This allows us to substitute $e_{i,\theta}(t)$ into \eqref{eq:errors} and consider the evolution of $\ve[i,c]{e}$ as a non-autonomous system on its own.

Introduce a candidate Lyapunov function $U:=\frac12\|\ve[i,c]{e}\|^2>0$ which satisfies
\begin{align*}
\dot{U}=&{}-k_{c}\,{\sech}^2(k_{s}e_{i,\theta})\tanh(k_t\|\ve[i,c]{e}\|)\|\ve[i,c]{e}\|+\nonumber\\
{}&-\frac{k_\theta r_c}{2}\frac{\sin e_{i,\theta}}{\sqrt{1+\cos e_{i,\theta}}}\left\langle\ve[i,c]{e},\begin{bmatrix}
-\sin (\theta_{i,d}-e_{i,\theta})\\
\cos (\theta_{i,d}-e_{i,\theta})
\end{bmatrix}\right\rangle.
\end{align*}
Let $V=\sqrt{U}=\|\ve[i,c]{e}\|/\sqrt{2}$, yielding $\dot{V}=\dot{U}/(2V)$, \ie
\begin{align*}
 \dot{V}=&{}-\frac{k_{c}}2\,{\sech}^2(k_{s}e_{i,\theta})\tanh(k_t V)\nonumber\\
{}&-\frac{k_\theta r_c}4\frac{\sin e_{i,\theta}}{\sqrt{1+\cos e_{i,\theta}}}\left\langle\frac{\ve[i,c]{e}}{V},\begin{bmatrix}
-\sin (\theta_{i,d}-e_{i,\theta})\\
\cos (\theta_{i,d}-e_{i,\theta})
\end{bmatrix}\right\rangle\\
\leq&{}-\frac{k_{c}}2\,{\sech}^2(k_{s}e_{i,\theta})\tanh(k_t V)+\frac{k_\theta r_c}{2\sqrt{2}}\frac{\sin |e_{i,\theta}|}{\sqrt{1+\cos e_{i,\theta}}},
\end{align*}
where we used the Cauchy-Schwarz inequality.

By definition of $e_{i,\theta}\rightarrow0$ as $t\rightarrow\infty$, for any $\delta>0$ there is a time $T$ such that $|e_{i,\theta}(t)|\leq \delta$ for all $t\geq T$. In particular, for any small $\delta\in(0,1)$ there is a $T(\delta)$ such that
\begin{align*}
\dot{V}\leq-\frac{k_{c}}2\,\tanh(k_t V)+\frac{k_\theta r_c}{2\sqrt{2}}\delta,
\end{align*}
for all $t\geq T(\delta)$. Introduce a variable $W$ given by
\begin{align*}
\dot{W}:=-\frac{k_{c}}2\,\tanh(k_t W)+\frac{k_\theta r_c}{2\sqrt{2}}\delta,
\end{align*}
Let $V(0)\leq W(0;\delta)$, then $V(t)\leq W(t;\delta)$ by the comparison principle \cite{khalil}. For small $\delta$, it can be shown that
\begin{align*}
\lim_{t\rightarrow\infty}W(t;\delta)=\frac{1}{k_t}{\mathrm{atanh}}\,\frac{k_\theta r_c\delta}{\sqrt{2}k_c}.
\end{align*}
Note that $\lim_{\delta\rightarrow0} W(t;\delta)=0$. It follows that for sufficiently large $t$, $V(t)$ is bounded above by an arbitrarily small number and bounded below by $0$. Hence $\lim_{t\rightarrow\infty}V(t)=0$.\end{proof}

\begin{remark}As desired, $(\ve[i,c,d]{x},\theta_d)$ is an almost globally asymptotically stable equilibrium of the UAV system
\begin{align*}
\vd[i,c]{x}=\ve[i,c]{u},\quad\dot{\theta}_i=u_{i,\theta},
\end{align*}
in the case of a constant $\ve[i,c,d]{x}$. This means that the dynamics \eqref{eq:errors}--\eqref{eq:etheta} and our preceding analysis concerns only one maneuver from $(\ve[i,c]{x},\theta_i)$ to $(\ve[i,c,d]{x},\theta_{i,d})$. Note that in practice, we have non-constant $(\ve[i,c,d]{x},\theta_{i,d})$ and require $\ve[i,c]{x}$ to converge to a set, the camera field of view ${\mathcal{V}}_i$, rather than a point $\ve[i,c,d]{x}$. The field of view ${\mathcal V}_i$ will contain the desired equilibrium $\ve[i,c,d]{x}$ in finite time. Then, as $\ve[i,c,d]{x}$ is updated, ${\mathcal{V}}_i$ will move towards a new set containing the next equilibrium.
\end{remark}

\subsection{Nano-UAV swarm experiment}

The combined SLAM, ICP, and exploration algorithms are implemented on a swarm of three Crazyflie 2.1 nano-UAVs. The environment that the swarm operates consists of two adjacent rooms. The UAVs start out close to each other, aligned with a nearby wall in the environment. The nearby wall provides a source of LRF sensor readings that the SLAM and ICP algorithms can exploit for initial synchronization of the three SLAM maps. The UAVs are stabilized to different heights as to not interfere with each other's LRF.


The SLAM map and trajectories taken by the UAVs are displayed in Fig. \ref{fig:slam}. The UAVs start out close to each other on a line and are by default located in the center of their maps. After the initial SLAM and ICP phases, the UAVs start to explore the environment. The local view that each UAV has of its environment is displayed in the map. The UAV detects the walls of the environment and explores both its inside and  boundaries.  Two of the UAVs stay in the first room. The third UAV discovers the adjacent second room and moves into it for exploration. 

The camera coverage map $\ve[i]{m}$ with the timestamped field of view ${\mathcal{V}}_i$  of the three UAVs is displayed in Fig. \ref{fig:timestamp}. The three timestamp maps are different since the belief that each UAV has about the positions of the other members of the swarm differs somewhat, see Fig. \ref{fig:slam}. Fig. \ref{fig:timestamp} also displays the total coverage obtained as the percentage of all strictly positive map elements $(\ve[i]{m})_{ab}>0$. The increase in coverage is approximately linear. Note that the camera coverage maps does not account for walls in the environment.

\begin{figure}[htp!]
\centering
\includegraphics[width=1\linewidth]{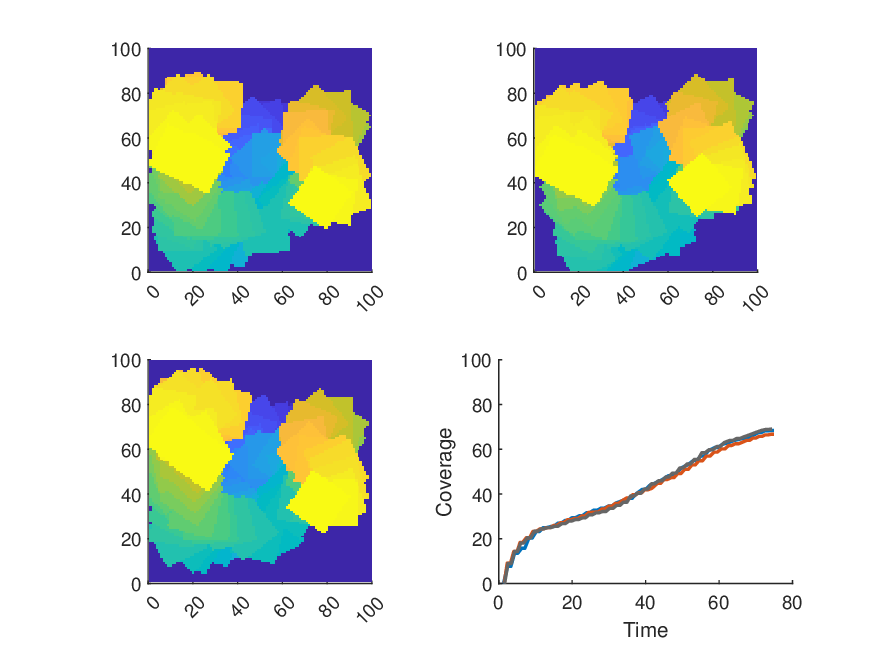}
 \caption{The camera coverage maps $\ve[i]{m}$ for the three UAVs. The total coverage of $\ve[i]{m}$ as a function of time (bottom right). The coverage is similar for each UAV, the colors of the three curves are the same as those in Fig. \ref{fig:slam}.}
\label{fig:timestamp}
\end{figure}

\section{Conclusions}

This paper is a proof of concept study that demonstrates SLAM and exploration in a swarm of three nano-UAVs. In particular, SLAM has not been implemented onboard the Crazyflie 2.1 platform previously, let alone in a swarm. We couple the SLAM algorithm with an exploration algorithm that has been modified for implementation on the limited hardware of the experimental platform. The two algorithms synergize; the SLAM algorithm maps out new areas and the exploration algorithm moves the UAV swarm into these areas from where they can map out further areas. The agents interact and solve the exploration task in a cooperative fashion. We also prove an almost global stability result for the our modified UAV position and orientation control algorithm.

\bibliographystyle{ieeetr} 
\bibliography{mybib.bib}

\begin{thebibliography}{10}

\bibitem{saska2016swarm}
M.~Saska, V.~Von{\'a}sek, J.~Chudoba, J.~Thomas, G.~Loianno, and V.~Kumar,
  ``Swarm distribution and deployment for cooperative surveillance by
  micro-aerial vehicles,'' {\em Journal of Intelligent \& Robotic Systems},
  vol.~84, no.~1, pp.~469--492, 2016.

\bibitem{mcguire2019minimal}
K.~McGuire, C.~De~Wagter, K.~Tuyls, H.~Kappen, and G.~de~Croon, ``Minimal
  navigation solution for a swarm of tiny flying robots to explore an unknown
  environment,'' {\em Science Robotics}, vol.~4, no.~35, 2019.

\bibitem{duisterhof2021sniffy}
B.~Duisterhof, S.~Li, J.~Burgu{\'e}s, and G.~Reddi, V.J.and de~Croon, ``Sniffy
  bug: A fully autonomous swarm of gas-seeking nano quadcopters in cluttered
  environments,'' in {\em Proceedings of the 34th IEEE/RSJ International
  Conference on Intelligent Robots and Systems}, pp.~9099--9106, 2021.

\bibitem{steux2010tinySLAM}
B.~Steux and O.~Hamzaoui, ``tiny{SLAM}: A {SLAM} algorithm in less than 200
  lines {C}-language program,'' in {\em Proceedings of the 11th International
  Conference on Control Automation Robotics Vision}, pp.~1975--1979, 2010.

\bibitem{ganganath2016distributed}
N.~Ganganath, C.-T. Cheng, and K.~Chi, ``Distributed antiflocking algorithms
  for dynamic coverage of mobile sensor networks,'' {\em IEEE Transactions on
  Industrial Informatics}, vol.~12, no.~5, pp.~1795--1805, 2016.

\bibitem{coppola2020survey}
M.~Coppola, K.~McGuire, C.~De~Wagter, and G.~De~Croon, ``A survey on swarming
  with micro air vehicles: Fundamental challenges and constraints,'' {\em
  Frontiers in Robotics and AI}, vol.~7, p.~18, 2020.

\bibitem{mcguire2019comparative}
K.~McGuire, G.~de~Croon, and K.~Tuyls, ``A comparative study of bug algorithms
  for robot navigation,'' {\em Robotics and Autonomous Systems}, vol.~121,
  p.~103261, 2019.

\bibitem{gabriely2001spanning}
Y.~Gabriely and E.~Rimon, ``Spanning-tree based coverage of continuous areas by
  a mobile robot,'' {\em Annals of Mathematics and Artificial Intelligence},
  vol.~31, no.~1, pp.~77--98, 2001.

\bibitem{galceran2013survey}
E.~Galceran and M.~Carreras, ``A survey on coverage path planning for
  robotics,'' {\em Robotics and Autonomous systems}, vol.~61, no.~12,
  pp.~1258--1276, 2013.

\bibitem{hazon2005redundancy}
N.~Hazon and G.~Kaminka, ``Redundancy, efficiency and robustness in multi-robot
  coverage,'' in {\em Proceedings of the 22th IEEE International Conference on
  Robotics and Automation}, pp.~735--741, 2005.

\bibitem{burgard2005coordinated}
W.~Burgard, M.~Moors, C.~Stachniss, and F.~E. Schneider, ``Coordinated
  multi-robot exploration,'' {\em IEEE Transactions on Robotics}, vol.~21,
  no.~3, pp.~376--386, 2005.

\bibitem{soria2021predictive}
E.~Soria, F.~Schiano, and D.~Floreano, ``Predictive control of aerial swarms in
  cluttered environments,'' {\em Nature Machine Intelligence}, vol.~3, no.~6,
  pp.~545--554, 2021.

\bibitem{olfati2006flocking}
R.~Olfati-Saber, ``Flocking for multi-agent dynamic systems: Algorithms and
  theory,'' {\em IEEE Transactions on Automatic Control}, vol.~51, no.~3,
  pp.~401--420, 2006.

\bibitem{sobreira2019Map-matching}
H.~Sobreira, C.~Costa, I.~Sousa, L.~Rocha, J.~Lima, P.~Farias, P.~Costa, and
  A.~Moreira, ``Map-matching algorithms for robot self-localization: A
  comparison between perfect match, iterative closest point and normal
  distributions transform,'' {\em Journal of Intelligent \& Robotic Systems},
  vol.~93, 03 2019.

\bibitem{lee2012exponential}
T.~Lee, ``Exponential stability of an attitude tracking control system on so
  (3) for large-angle rotational maneuvers,'' {\em Systems \& Control Letters},
  vol.~61, no.~1, pp.~231--237, 2012.

\bibitem{vidyasagar1980decomposition}
M.~Vidyasagar, ``Decomposition techniques for large-scale systems with
  nonadditive interactions: Stability and stabilizability,'' {\em IEEE
  Transactions on Automatic Control}, vol.~25, no.~4, pp.~773--779, 1980.

\bibitem{khalil}
H.~Khalil, {\em Nonlinear Systems}.
\newblock Prentice Hall, 2002.

\end{thebibliography}


\begin{thebibliography}{1}

\bibitem{breskin1935synthetic}
Charles~A Breskin.
\newblock Synthetic plastics.
\newblock {\em Industrial \& Engineering Chemistry}, 27(10):1140--1142, 1935.

\end{thebibliography}

\end{document}